%% file: main.tex
\documentclass{article}

\usepackage[utf8]{inputenc} 
\usepackage[T1]{fontenc}    
\usepackage{hyperref}       
\usepackage{url}            
\usepackage{booktabs}       
\usepackage{amsfonts}       
\usepackage{nicefrac}       
\usepackage{microtype}      

\usepackage{subfig}
\usepackage{times}
\usepackage{anysize}  
\usepackage{graphicx} 
\usepackage{caption}
\usepackage[shortlabels]{enumitem}
\usepackage{placeins}
\usepackage{natbib}
\usepackage{hyperref}
\usepackage{epstopdf}

\usepackage{algorithm}
\usepackage{algorithmic}
\usepackage{amsfonts}
\usepackage{amsthm,amssymb}
\usepackage{amsmath}
\newcommand\floor[1]{\lfloor#1\rfloor}
\newcommand\ceil[1]{\lceil#1\rceil}

\newcommand{\mT}{\mathcal{T}}
\newcommand{\mL}{\mathcal{L}}
\newcommand{\mC}{\mathcal{C}}
\newcommand{\mR}{\mathcal{R}}

\newcommand{\Prob}{\mathbb{P}}

\newcommand{\E}{\mathbb{E}}
\newcommand{\indfunc}{\mathbb{I}}
\newtheorem{theorem}{Theorem}

\newtheorem{corollary}{Corollary}
\newtheorem{lemma}{Lemma}

\newtheorem{assu}{Assumption}
\newtheorem{example}{Example}

\title{\LARGE{\textbf{Nearly Optimal Algorithms for Piecewise-Stationary Cascading Bandits}}}
\author{ 
	Lingda Wang\textsuperscript{1}\thanks{indicates equal contributions.}, Huozhi Zhou\textsuperscript{1}\footnotemark[1], Bingcong Li\textsuperscript{2}, Lav R. Varshney\textsuperscript{1}, Zhizhen Zhao\textsuperscript{1} 	\vspace{0.1cm} \\
	\textsuperscript{1}{\normalsize\textit{ECE Department and CSL, University of Illinois at Urbana-Champaign, Urbana, IL 61801, USA}} \\	 
	\texttt{\normalsize\{lingdaw2, hzhou35, varshney, zhizhenz\}@illinois.edu} \\
	\textsuperscript{2}{\normalsize\textit{ECE Department, University of Minnesota - Twin Cities, Minneapolis, MN 55455, USA}} \\
	\normalsize\texttt{lixx5599@umn.edu} \\	
}
\vspace{0.3cm}

\begin{document}

\maketitle
\begin{abstract}
Cascading bandit (CB) is a popular model for web search and online advertising, where an agent aims to learn the $K$ most attractive items out of a ground set of size $L$ during the interaction with a user. However, the stationary CB model may be too simple to apply to real-world problems, where user preferences may change over time. Considering piecewise-stationary environments, two efficient algorithms, \texttt{GLRT-CascadeUCB} and \texttt{GLRT-CascadeKL-UCB}, are developed and shown to ensure regret upper bounds on the order of $\mathcal{O}(\sqrt{NLT\log{T}})$, where $N$ is the number of piecewise-stationary segments, and $T$ is the number of time slots. At the crux of the proposed algorithms is an almost parameter-free change-point detector, the generalized likelihood ratio test (GLRT). Comparing with existing works, the GLRT-based algorithms: i) are free of change-point-dependent information for choosing parameters; ii) have fewer tuning parameters; iii) improve at least the $L$ dependence in regret upper bounds. In addition, we show that the proposed algorithms are optimal (up to a logarithm factor) in terms of regret by deriving a minimax lower bound on the order of $\Omega(\sqrt{NLT})$ for piecewise-stationary CB. The efficiency of the proposed algorithms relative to state-of-the-art approaches is validated through numerical experiments on both synthetic and real-world datasets. 
\end{abstract}

\input{Introduction}


\input{Setting}

\input{Algorithm}

\input{Analysis}

\input{Experiments}

\input{Conclusion}
\bibliographystyle{apalike}
\bibliography{cas_ban_ref}
%

\input{Appendix}

\end{document}

%% file: Introduction.tex
\section{Introduction}
\label{sec:Intro}

Online recommendation~\citep{li2016collaborative} and web search~\citep{dupret2008user,zoghi2017online} are of significant importance in the modern economy. Based on a user's browsing history, these systems strive to maximize satisfaction and minimize regret by presenting the user with a list of items (e.g., web pages and advertisements) that meet her/his preference. Such a scenario can be modeled via cascading bandits (CB)~\citep{kveton2015cascading}, where an agent aims to identify the $K$ most attractive items out of total $L$ items contained in the ground set. The learning task proceeds sequentially, where per time slot, the agent recommends a ranked list of $K$ items and receives the reward and feedback on which item is clicked by the user.

CB can be viewed as multi-armed bandits (MAB) tailored for cascade model (CM)~\citep{craswell2008experimental}, where CM models a user’s online behavior. Existing works on CB~\citep{kveton2015cascading,cheung2019thompson} and MAB~\citep{lai1985asymptotically,auer2002finite,li2019bandit} can be categorized according to whether stationary or non-stationary environment is studied. In stationary environments, the reward distributions of arms (in MAB) or the attraction distributions of items (in CB) do not evolve over time. On the other hand, non-stationary environments are prevalent in real-world applications such as web search, online advertisement, and recommendation since user's preference is time-varying~\citep{jagerman2019people,yu2009piecewise,pereira2018analyzing}. Algorithms designed for stationary scenarios can suffer from a linear regret when applied to non-stationary environments directly~\citep{li2019cascading,garivier2011upper}. The most common non-stationary environments include adversarial~\citep{auer2002nonstochastic,littlestone1994weighted}, piecewise-stationary~\citep{hartland2007change,kocsis2006discounted,garivier2011upper}, and slow-varying environment~\citep{besbes2014stochastic}. More recently, the stochastic environment with adversarial corruptions and the stochastically-constrained
adversarial environment are proposed by~\citet{lykouris2018stochastic} and~\citet{wei2018more}, respectively, which are mixtures of both stochastic and adversarial environments. Other interesting non-stationary bandit models can be found in~\citet{whittle1988restless},~\citet{cella2019stochastic}, etc. 

In this paper, we focus on the piecewise-stationary environment, where the user's preference remains stationary over some number of time slots, named \emph{piecewise-stationary segments}, but can shift abruptly at some unknown times, called \emph{change-points}. Piecewise-stationary models provide more accurate characterization of real-world applications. For instance, in recommendation systems, user's preference for an item is neither invariant nor changing per time slot.

To address the piecewise-stationary MAB, two types of approaches have been proposed in the literature: \emph{passively adaptive approaches}~\citep{garivier2011upper,besbes2014stochastic,wei2018abruptly} and \emph{actively adaptive approaches}~\citep{cao2019nearly,liu2018change,besson2019generalized,auer2019adaptively}. Passively adaptive approaches ignore when a change-point occurs. For active adaptive approaches, a change-point detection algorithm such as CUSUM~\citep{page1954continuous,liu2018change}, Page Hinkley Test (PHT)~\citep{hinkley1971inference,liu2018change}, or comparing running sample means over a sliding window (CMSW)~\citep{cao2019nearly} is included. Within the area of piecewise-stationary CB, only passively adaptive approaches have been studied~\citep{li2019cascading}. In this context, we introduce the generalized likelihood ratio test (GLRT)~\citep{willsky1976generalized,besson2019generalized}
for actively adaptive CB algorithms. In particular, we develop two GLRT based algorithms \texttt{GLRT-Cascade-UCB} and \texttt{GLRT-CascadeKL-UCB} to enhance both theoretical and practical effectiveness for piecewise-stationary CB. The merits of the proposed algorithms are summarized as follows
\begin{enumerate}
    
    \item \textbf{Practically oriented.} The proposed GLRT based algorithms are more practical than previous CMSW or CUSUM based works~\citep{liu2018change,cao2019nearly}, since: i) no change-point-dependent parameter is required by GLRT; ii) have fewer tuning parameters.
    \item \textbf{Tighter regret bounds.}  $\mathcal{O}(\sqrt{NLT\log{T}})$ regret bounds of both proposed algorithms are established, where $L$ is the number of items and $T$ is the number of time slots. 
    Our regret bound tightens those of~\citet{li2019cascading} by a factor of $\sqrt{L}$ and $\sqrt{L}\log T$, respectively.
    \item \textbf{Lower-bound matching.} We establish that the minimax regret lower bound for piecewise-stationary CB is $\Omega(\sqrt{NLT})$. Such a lower bound: i) implies the proposed algorithms are optimal up to a logarithm factor; ii) is the first to characterize dependence on $N$, $L$, and $T$ for piecewise-stationary CB.
    \item \textbf{Numerically attractive.} On both synthetic and real-world datasets, numerical experiments reveal the merits of  proposed algorithms over state-of-the-art approaches. 
\end{enumerate}

The remainder of the paper is organized as follows. We describe the problem formulation in Section~\ref{sec:setting}. The GLRT change-point detector together with the proposed algorithms, \texttt{GLRT-CascadeUCB} and \texttt{GLRT-CascadeKL-UCB}, are detailed in Section~\ref{sec:algorithms}. We prove upper bounds on the regret of the proposed algorithms and the minimax regret lower bound in Section~\ref{sec:analysis}. Numerical experiments are presented in Section~\ref{sec:Experiments}. Finally, we conclude the paper in Section~\ref{sec:Conclusion}. 

%% file: Setting.tex
\section{Problem Formulation}
\label{sec:setting}
This section first reviews the CM and CB in Section~\ref{sec:cmcb}, and then introduces the piecewise-stationary CB in Section~\ref{sec:pscb}.  
\subsection{Cascade Model and Cascading Bandits}
\label{sec:cmcb}

CB~\citep{kveton2015cascading}, as a learning variant of CM, depicts the interaction between the agent and the user on $T$ time slots, where the user's preference is learned. CM~\citep{craswell2008experimental} explains the user's behavior in a specific time slot $t$.  

In CM, the user is presented with a $K$-item ranked list $\mathcal{A}_t := \left(a_{1,t},\ldots, a_{K,t}\right)\in\Pi_K\left(\mathcal{L}\right)$ from $\mL$ at time slot $t$, where $\mL:=\{1,2,\ldots,L\}$ is a ground set containing $L$ items (e.g., web pages or advertisements), and $\Pi_K\left(\mathcal{L}\right)$ is the set of all $K$-permutations of the ground set $\mathcal{L}$. CM can be parameterized by the attraction probability vector $\mathbf{w}_t = \left[\mathbf{w}_t(1),\ldots,\mathbf{w}_t(L)\right]^\top\in[0,1]^L$. The user browses the list $\mathcal{A}_t$ from the first item $a_1$ in order, and each item $a_k$ attracts the user to click it with probability $\mathbf{w}_t(a_k)$. The user will stop the process after clicking the first attractive item. In particular, when an item $a_{k,t}$ is clicked, it means that i) items from $a_{1,t}$ to $a_{k-1,t}$ are not attractive to the user; and ii) items $a_{k+1,t}$ to $a_{K,t}$ are not browsed so whether they are attractive to the user is unknown. Clearly, if no item is attractive, the user will browse the whole list and click on nothing. Note that CM can be generalized to multi-click cases \citep{wang2015incorporating,yue2010learning}, but this is beyond the scope of this paper.

Building upon CM, a CB can be described by a tuple $(\mathcal{L},\,\mathcal{T},\,\{f_{\ell,t}\}_{\ell\in\mathcal{L},t\in\mathcal{T}},K)$, where $\mathcal{T}:=\{1,2,\ldots,T\}$ collects all $T$ time slots. 
Whether the user is attracted by item $\ell$ at time slot $t$ is actually a Bernoulli random variable $Z_{\ell,t}$, whose pmf is $f_{\ell,t}$. As convention, $Z_{\ell,t} = 1$ indicates item $\ell$ is attractive to the user. We also denote $\mathbf{Z}_t:=\left\{Z_{\ell,t}\right\}_{\ell\in\mathcal{L}}$ as all the attraction variables of the ground set. Clearly, the $\{f_{\ell,t }\}_{\ell\in\mathcal{L},t\in\mathcal{T}}$ are parameterized by the attraction probability vectors $\{\mathbf{w}_t\}_{t\in\mathcal{T}}$, which are unknown to the agent. Since CB is designed for stationary environments, the attraction probability vector $\mathbf{w}_t$ is time-invariant, and thus can be further simplified as $\mathbf{w}$. CB poses a mild assumption on $\{f_{\ell,t }\}_{\ell\in\mathcal{L},t\in\mathcal{T}}$ for simplicity.

\begin{assu}
\label{assump:ind}
The attraction distributions $\{f_{\ell,t}\}_{\ell\in\mathcal{L},t\in\mathcal{T}}$ are independent both across items and time slots.
\end{assu}
Per slot $t$, the agent recommends a list of $K$ items $\mathcal{A}_t$ to the user based on the feedback from the user up to time slot $t-1$. The feedback at time slot $t$ refers to the index of the clicked item, given by
\begin{align*}
  F_t =
  \begin{cases}
  \emptyset,\,&\mbox{if no click},\\
  \arg\min_k\{1\le k\le
  K:Z_{a_{k,t},t}=1\},\,&\mbox{otherwise}.\\
  \end{cases}
\end{align*}
After the user browses the list follows the protocol described by CM, the agent observes the feedback $F_t$. Along with $F_t$ is a zero-one reward indicating whether there is a click
\begin{equation}
\label{eq:reward}
     r\left(\mathcal{A}_t,\mathbf{Z}_t\right)=1-\prod_{k=1}^K\left(1-Z_{a_{k,t},t}\right),
\end{equation}
where $r\left(\mathcal{A}_t,\mathbf{Z}_t\right) = 0$ if $F_t = \emptyset$.
Then, this process proceeds to time slot $t+1$. The goal of the agent is to maximize the expected cumulative reward over the whole time horizon $\mT$. Noticing that $Z_{\ell,t}$s are independent, the expected reward at time slot $t$ can be computed as $\E\left[ r\left(\mathcal{A}_t,\mathbf{Z}_t\right)\right] = r\left(\mathcal{A}_t,\mathbf{w}\right)$. The optimal list $\mathcal{A}^*$ remains the same for all time slots, which is the list containing the $K$ most attractive items. 

\subsection{Piecewise-Stationary Cascading Bandits}
\label{sec:pscb}

The stationarity assumption on CB limits its applicability for real-world applications, as users tend to change their preferences as time goes on~\citep{jagerman2019people}. This fact leads to piecewise-stationary CB. Consider a piecewise-stationary CB with $N$ segments, where the attraction probabilities of items remain identical per segment. Mathematically, $N$ can be written as
\begin{equation}
\label{def:N}
    N = 1+\sum_{t=1}^{T-1}\mathbb{I}\{\exists \ell\in\mathcal{L}~\mbox{s.t.}~f_{\ell,t}\neq f_{\ell,t+1}\},
\end{equation}
where $\mathbb{I}\{\cdot\}$ is the indicator function, and a change-point is the time slot $t$ that satisfies $\exists \ell\in\mathcal{L}~\mbox{s.t.}~f_{\ell,t}\neq f_{\ell,t+1}$. Hence it is clear that there are $N-1$ change-points in the piecewise-stationary CB considered. These change-points are denoted by $\nu_1,\ldots,\nu_{N-1}$ in a chronological manner. Specifically, $\nu_0=0$ and $\nu_N=T$ are introduced for consistency. For the $i$th piecewise-stationary segment $t\in[\nu_{i-1}+1,\nu_i]$, $f_{\ell}^i$ and $\mathbf{w}^i(\ell)$ denote the attraction distribution and the expected attraction of item $\ell$, respectively, which are again unknown to the agent. Attraction probability vector $\mathbf{w}^i=[\mathbf{w}^i(1),\ldots,\mathbf{w}^i(L)]^\top$ is introduced to collect $\mathbf{w}^i(\ell)$s. 

In a piecewise-stationary CB, agent interactions are the same as CB. The agent's policy can be evaluated by its expected cumulative reward, or equivalently its expected cumulative regret:
\begin{equation}
\label{eq:regret}
    \mathcal{R}(T)=\mathbb{E}\left[\sum_{t=1}^T R\left(\mathcal{A}_t,\mathbf{w}_t,\mathbf{Z}_t\right)\right],
\end{equation}
where the expectation $\mathbb{E}[\cdot]$ is taken with respect to a sequence of $\mathbf{Z}_t$ and the corresponding $\mathcal{A}_t$. Here, $R(\mathcal{A}_t,\mathbf{w}_t,\mathbf{Z}_t)=r(\mathcal{A}^*_t,\mathbf{w}_t)-r(\mathcal{A}_t,\mathbf{Z}_t)$ is the regret at time slot $t$ with
\begin{equation*}
  \mathcal{A}^*_t=\arg\max_{\mathcal{A}_t\in\Pi_K(\mathcal{L})}  r\left(\mathcal{A}_t,\mathbf{w}_t\right)
\end{equation*}
being the optimal list that maximizes the expected reward at time slot $t$. The regret defined in \eqref{eq:regret} is also known as switching regret, which is widely adopted in piecewise-stationary bandits~\citep{kocsis2006discounted,garivier2011upper,liu2018change,cao2019nearly,besson2019generalized}. Since switching regret is measured with respect to the optimal piecewise-stationary policy, the optimal list $\mathcal{A}^*_t$ for each time slot is no longer time-invariant. This leads to a much harder algorithm design problem, since the non-stationary environment should be properly coped with.

%% file: Algorithm.tex
\section{Algorithms}
\label{sec:algorithms}

This section presents adaptive approaches for piecewise-stationary CB using an efficient change-point detector.

\subsection{Generalized Likelihood Ratio Test}
\begin{algorithm}
\footnotesize
\caption{GLRT Change-Point Detector: $\mbox{GLRT}(X_1,\ldots, X_n;\delta)$}
\label{alg:cd}
\begin{algorithmic}[1]
\REQUIRE observations $X_1, \ldots, X_n$ and confidence level $\delta$
\STATE Compute the GLR statistic $\mbox{GLR}(n)$ according to~\eqref{eq:GLR_STAT} and the threshold $\beta(n,\delta)$ according to~\eqref{eq:thre}
\IF{$\mbox{GLR}(n)\ge\beta(n,\delta)$} 
\STATE Return True
\ELSE 
\STATE Return False
\ENDIF
\end{algorithmic}
\end{algorithm}

As the adaptive approach is adopted in this paper, a brief introduction about change-point detection is given in this subsection. Sequential change-point detection is of fundamental importance in statistical sequential analysis, see e.g., \citep{hadjiliadis2006optimal,siegmund2013sequential,draglia1999multihypothesis,siegmund1995using,lorden1971procedures,moustakides1986optimal}. However, the aforementioned approaches typically rely on the knowledge of either pre-change or post-change distribution, rendering barriers for the applicability in piecewise-stationary CB.

In general, with pre-change and post-change distributions unknown, developing algorithms with provable guarantees is challenging. Our solution relies on the GLRT that is summarized under Algorithm~\ref{alg:cd}. Compared with existing change-point detection methods that have provable guarantees \citep{liu2018change,cao2019nearly}, advantages of GLRT are threefold: i) \emph{Fewer tuning parameters}. The only required parameter for GLRT is the confidence level of change-point detection $\delta$, while CUSUM~\citep{liu2018change} and CMSW~\citep{cao2019nearly} have three and two parameters to be manually tuned, respectively. ii) \emph{Less prior knowledge needed}. GLRT does not require the information on the smallest magnitude among the change-points, which is essential for CUSUM. iii) \emph{Better performance}. The GLRT is more efficient than CUSUM and CMSW in the averaged detection time. As shown in the numerical experiments in Example~\ref{exp:1}, GLRT has approximately $20\%$ and $50\%$ improvement over CUSUM and CMSW, respectively.

Next, the GLRT is formally introduced. Suppose we have a sequence of Bernoulli random variables $\{X_t\}_{t=1}^n$ and aim to determine if a change-point exists as fast as we can. This problem can be formulated as a parametric sequential test of the following two hypotheses: 
\begin{align*}
&\mathcal{H}_0:\exists \mu_0:X_1,\ldots,X_n\overset{\textup{i.i.d}}{\sim}~\mbox{Bern}(\mu_0),\\
&\mathcal{H}_1:\exists \mu_0\neq \mu_1,\, \tau\in[1,n-1]:X_1,\ldots,X_\tau\overset{\textup{i.i.d}}{\sim}~\mbox{Bern}(\mu_0)~\mbox{and}~X_{\tau+1},\ldots,X_{n}\overset{\textup{i.i.d}}{\sim}\mbox{Bern}(\mu_1),
\end{align*}
where $\textup{Bern}(\mu)$ is the Bernoulli distribution with mean $\mu$. The GLR statistic is 
\begin{align}
\label{eq:GLR_STAT}
    \mbox{GLR}(n)=\sup_{s\in[1,n-1]}&[s\times\text{KL}\left(\hat{\mu}_{1:s},\hat{\mu}_{1:n}\right)+(n-s)\times\text{KL}\left(\hat{\mu}_{s+1:n},\hat{\mu}_{1:n}\right)],
\end{align}
where $\hat{\mu}_{s:s'}$ is the empirical mean of observations from $X_s$ to $X_{s'}$, and $\mbox{KL}(x,y)$ is the Kullback–Leibler (KL) divergence of two Bernoulli distributions,
\begin{equation*}
    \mbox{KL}(x,y) = x\log\left(\frac{x}{y}\right)+(1-x)\log\left(\frac{1-x}{1-y}\right).
\end{equation*}

By comparing $\mbox{GLR}(n)$ in \eqref{eq:GLR_STAT} with the threshold $\beta(t,\delta)$, one can decide whether a change-point appears for a length $n$ sequence, where 
\begin{equation}
\label{eq:thre}
    \beta(t,\delta)=2\mathcal{G}\left(\frac{\log(3t\sqrt{t}/\delta)}{2}\right)+6\log(1+\log{t}),
\end{equation}
and $\mathcal{G}(\cdot)$ has the same definition as that in (13) of~\citet{kaufmann2018mixture}. The choice of $\delta$ is influences the sensitivity of the GLRT. For example, a larger $\delta$ makes the GLRT response faster to a change-point, but increases the probability of false alarm.

The efficiency of a change-point detector for a length $n$ sequence is evaluated via its detection time,
\begin{equation*}
 \tau=\inf\{t\le n:\textup{GLR}(t)\ge\beta(t,\delta)\}.  
\end{equation*}

To better understand the performance of GLRT against CUSUM ans CMSW, it is instructive to use an example.

\begin{example}[Efficiency of GLRT]
\label{exp:1}
Consider a sequence of Bernoulli random variables $\{X_t\}_{t=1}^{n}$ with $n=4000$, where $X_1,\cdots,X_{2000}$ are generated from \textup{Bern(0.2)} and the remaining ones are generated from \textup{Bern(0.8)}, as shown in Figure~\ref{fig:expample1_tra} (red line). By setting $\delta = 1/n$ for GLRT and choosing parameters of CUSUM and CMSW as recommended in~\citet{liu2018change} and~\citet{cao2019nearly}, the average detection times after 100 Monte Carlo trials are $2024.55\pm 6.8451$ (GLRT, green line), $2030.25\pm6.74$ (CUSUM, blue line), and $2045.59\pm 4.48$ (CMSW, black line), respectively. In a nutshell, GLRT improves about $20\%$ over CUSUM and $50\%$ over CMSW.
\end{example} 

\begin{figure}[htb]
	\begin{center}
	\includegraphics[width=0.8 \linewidth]{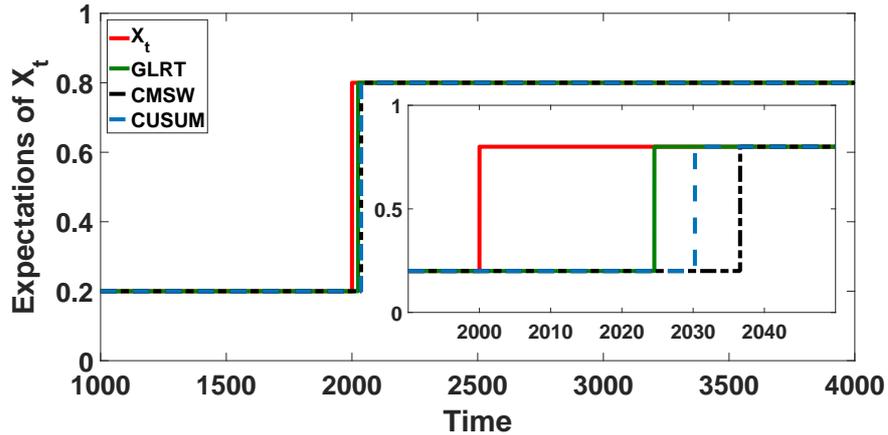}
	\end{center}
	\caption{Expectations of $X_t$'s with $n=4000$ and expected detection time of GLRT, CUSUM, and CMSW.}
	\label{fig:expample1_tra}
\end{figure}

\subsection{The GLRT Based CB Algorithms}
Leveraging GLRT as the change-point detector, the proposed algorithms, \texttt{GLRT-CascadeUCB} and \texttt{GLRT-CascadeKL-UCB}, are presented in Algorithm~\ref{alg:bandit}. On a high level, three phases comprise the proposed algorithms.  \\ \emph{Phase 1}: The forced uniform exploration to ensure that sufficient samples are gathered for all items to perform the GLRT detection (Algorithm~\ref{alg:cd}).\\ \emph{Phase 2}: The UCB-based exploration (UCB or KL-UCB) to learn the optimal list on each piecewise-stationary segment. \\
\emph{Phase 3}: The GLRT change-point detection (Algorithm~\ref{alg:cd}) to monitor if global restart should be triggered. 

Besides the time horizon $\mathcal{T}$, the ground set $\mathcal{L}$, the number of items in list $K$, the proposed algorithms only require two parameters $p$ and $\delta$ as inputs. The probability $p$ is used to control the portion of uniform exploration in Phase 1, and it appears also in other bandit algorithms for piecewise-stationary environments~\citep{liu2018change,cao2019nearly}. While the confidence level $\delta$ is the only parameter required by GLRT. Hence, the proposed algorithms are more practical compared with existing algorithms\citep{liu2018change,cao2019nearly}, since: i) no prior knowledge on change-point-dependent parameter is needed; ii) fewer parameters are required. The choices of $\delta$ and $p$ will be clear in Section~\ref{sec:analysis}.

In Algorithm 2, we denote the last detection time as $\tau$. From slot $\tau$ to current slot, let $n_l$ denote the number of observations for $\ell$th item, and $\hat{\mathbf{w}}(\ell)$ its corresponding sample mean. The algorithm determines whether to perform a uniform exploration or a UCB-based exploration depending on whether line 4 of Algorithm~\ref{alg:bandit} is satisfied, which ensures the fraction of time slots performing the uniform exploration phase is about $p$. If the uniform exploration is triggered, the first item in the recommended list $\mathcal{A}_t$ will be item $a:=(t-\tau)\mod\floor{\frac{L}{p}}$, and the remaining items in the list are chosen uniformly at random (line 5), which ensures item $a$ will be observed by the user. If UCB-based exploration is adopted at time slot $t$, the algorithms will choose $K$ items (line 7) with $K$ largest UCB indices, 
\begin{equation}
\label{eq:A_t}
    \mathcal{A}_t = \arg\max_{\mathcal{A}\in\Pi_K(\mathcal{L})}\quad r\left(\mathcal{A},\mbox{UCB or }\textup{UCB}_{\textup{KL}}\right),
\end{equation}
which will be defined in~\eqref{eq:UCB1} and~\eqref{eq:KL-UCB}.
By recommending the list $\mathcal{A}_t$ and observing the user's feedback $F_t$ (line 9), we update the statistics (line 11) and perform the GLRT detection (line 12). If a change-point is detected, we set $n_\ell=0$ for all $\ell\in\mathcal{L}$, and $\tau=t$ (line 13). Finally, the UCB indices of each item are computed as follows (line 18),
\begin{align}
    \mbox{UCB}(\ell) &= \hat{\mathbf{w}}(\ell)+\sqrt{\frac{3\log(t-\tau)}{2n_\ell}},\label{eq:UCB1}\\
    \mbox{UCB}_{\textup{KL}}(\ell)&=\max\{q\in[\hat{\mathbf{w}}(\ell),1]:n_\ell\times\mbox{KL}(\hat{\mathbf{w}}(\ell),q)\le g(t-\tau)\},\label{eq:KL-UCB}
\end{align}
where $g(t)=\log{t}+3\log{\log{t}}$, and $\hat{\mathbf{w}}(\ell)=\frac{1}{n_\ell}\sum_{n=1}^{n_\ell}X_{\ell,n}$. Notice that \eqref{eq:UCB1} is the UCB indices of \texttt{GLRT-CascadeUCB}, and \eqref{eq:KL-UCB} is the UCB indices of \texttt{GLRT-CascadeKL-UCB}. For the intuitions behind, we refer the readers to Proof of Theorem 1 in~\citet{auer2002finite} and Proof of Theorem 2 in~\citet{cappe2013kullback}.
\begin{algorithm}
\footnotesize
\caption{The \texttt{GLRT-CascadeUCB} and \texttt{GLRT-CascadeKL-UCB} Algorithms}
\label{alg:bandit}
\begin{algorithmic}[1]
\REQUIRE The time horizon $\mathcal{T}$, the ground set $\mathcal{L}$, $K$, exploration probability $p>0$, and confidence level $\delta>0$
\STATE \textbf{Initialization: }$\tau \leftarrow 0$ and $n_\ell \leftarrow 0$, $\forall \ell \in \mathcal{L}$
\FORALL{$t=1,2,\ldots, T$}
\STATE $a \leftarrow (t-\tau)\mod\floor{\frac{L}{p}}$
\IF{$a\le L$} 
\STATE Choose $\mathcal{A}_t$ such that $a_{1,t}\leftarrow a$ and $a_{2,t},\ldots,a_{K,t}$ are chosen uniformly at random
\ELSE
\STATE Compute the list $\mathcal{A}_t$ follows~\eqref{eq:A_t}
\ENDIF
\STATE Recommend the list $\mathcal{A}_t $ to user, and observe feedback $F_t$
\FORALL{$k=1,\ldots,F_t$}
\STATE $\ell\leftarrow a_{k,t}$, $n_\ell\leftarrow n_\ell+1$, $X_{\ell,n_\ell}\leftarrow \mathbb{I}\{F_t=k\}$ and $\hat{\mathbf{w}}(\ell)=\frac{1}{n_\ell}\sum_{n=1}^{n_\ell}X_{\ell,n}$
\IF{$\mbox{GLRT}(X_{\ell,1},\ldots,X_{\ell,n_\ell};\delta)$ = True}
\STATE $n_\ell\leftarrow 0,\,\forall\ell\in\mathcal{L}$, and $\tau\leftarrow t$
\ENDIF
\ENDFOR
\FOR{$\ell = 1,\cdots,L$}
\IF{$n_\ell\neq 0$}
\STATE Compute $\textup{UCB}(\ell)$ according to~\eqref{eq:UCB1} for \texttt{GLRT-CascadeUCB} or $\textup{UCB}_{\textup{KL}}(\ell)$ according to~\eqref{eq:KL-UCB} for \texttt{GLRT-CascadeKL-UCB}
\ENDIF
\ENDFOR
\ENDFOR
\end{algorithmic}
\end{algorithm}

%% file: Analysis.tex
\section{Theoretical Results}
\label{sec:analysis}
The theoretical guarantees of the proposed algorithms, \texttt{GLRT-CascadeUCB} and \texttt{GLRT-CascadeKL-UCB}, will be derived in this section. Specifically, the upper bounds on the regret of both proposed algorithms are developed in Sections~\ref{sec:5.1} and~\ref{sec:5.2}. A minimax regret lower bound for piecewise-stationary CB is established in Section~\ref{subsec:lb}. We further discuss our theoretical findings in Section~\ref{sec:5.3}.

Without loss of generality, for the $i$th piecewise-stationary segment, the ground set $\mathcal{L}$ is first sorted in decreasing order according to attraction probabilities, that is $\mathbf{w}^i(s_i(1))\ge\mathbf{w}^i(s_i(2))\ge\cdots\ge\mathbf{w}^i(s_i(L))$, for all $s_i(\ell)\in\mathcal{L}$. The optimal list at $i$th segment is thus all the permutations of the list $\mathcal{A}^*_i=\{s_i(1),\ldots,s_i(K)\}$. The item $\ell^*$ is optimal if $\ell^*\in\{s_i(1),\ldots,s_i(K)\}$, otherwise an item $\ell$ is called suboptimal. To simplify the exposition, the gap between the attraction probabilities of the suboptimal item $\ell$ and the optimal item $\ell^*$ at $i$th segment is defined as:
\begin{equation*}
\Delta^i_{\ell,\ell^*}=\mathbf{w}^i(\ell^*)-\mathbf{w}^i(\ell).
\end{equation*}
Similarly, the largest amplitude change among items at change-point $\nu_i$ is defined as
\begin{equation*}
\Delta^i_{\text{change}}=\max_{\ell\in\mathcal{L}}\left|\mathbf{w}^{i+1}(\ell)-\mathbf{w}^i(\ell)\right|,  
\end{equation*}
with $\Delta^0_{\text{change}}=\max_{\ell\in\mathcal{L}}\left|\mathbf{w}^1(\ell)\right|$. We have the following assumption for the theoretical analysis.
\begin{assu}
\label{assump:segment}
Define $d_i = d_i\left(p,\delta\right)=\ceil{\frac{4L\beta(T,\delta)}{p(\Delta^i_\textup{change})^2}+\frac{L}{p}}$ and assume $\nu_i-\nu_{i-1}\ge 2\max\{d_i,d_{i-1}\}$, $\forall i=1,\ldots,N-1$, with $\nu_N-\nu_{N-1}\ge2d_{N-1}$.
\end{assu}
Note that Assumption~\ref{assump:segment} is standard in a piecewise-stationary environment, and identical or similar assumptions are made in other change-detection based bandit algorithms~\citep{liu2018change,cao2019nearly,besson2019generalized} as well. It requires the length of the piecewise-stationary segment between two change-points to be large enough. Assumption~\ref{assump:segment} guarantees that with high probability all the change-points are detected within the interval $[\nu_i+1,\nu_i+d_i]$, which is equivalent to saying all change-points are detected correctly (low probability of false alarm) and quickly (low detection delay). This result is formally stated in Lemma~\ref{lem:cond_prob}. In our numerical experiments, the proposed algorithms work well even when Assumption~\ref{assump:segment} does not hold (see Section~\ref{sec:Experiments}).

\subsection{Regret Upper Bound for \texttt{GLRT-CascadeUCB}}

\label{sec:5.1}

Upper bound on the regret of \texttt{GLRT-CascadeUCB} is as follows.
\begin{theorem} 
\label{thm:1}
Suppose that Assumptions~\ref{assump:ind} and~\ref{assump:segment} are satisfied, \textup{\texttt{GLRT-CascadeUCB}} guarantees
\begin{equation*}
    \mathcal{R}(T)\le\underbrace{\sum_{i=1}^N\widetilde{C}_i}_{(a)}+\underbrace{Tp}_{(b)}+\underbrace{\sum_{i=1}^{N-1}d_i}_{(c)}+\underbrace{3NTL\delta}_{(d)},
\end{equation*}
where $\widetilde{C}_i=\sum_{\ell=K+1}^L\frac{12}{\Delta^i_{s_i(\ell),s_i(K)}}\log{T}+\frac{\pi^2}{3}L$.
\end{theorem}
\begin{proof}
The theorem is proved in Appendix~\ref{sec:proof_thm1}.
\end{proof}
Theorem~\ref{thm:1} indicates that the upper bound on the regret of \texttt{GLRT-CascadeUCB} is incurred by two types of costs that are further decomposed into four terms. Terms (a) and (b) upper bound the costs of UCB-based exploration and uniform exploration, respectively. The costs incurred by the change-point detection delay and the incorrect detections are bounded by terms (c) and (d). Corollary~\ref{col:1} follows directly from Theorem~\ref{thm:1}. 

\begin{corollary}
\label{col:1}
Let $\Delta^{\textup{min}}_{\textup{change}} = \min_{i\le N-1}\Delta^i_{\textup{change}}$ denote the smallest magnitude of any change-point on any item, and $\Delta^{\textup{min}}_{\textup{opt}} = \min_{i\le N}\Delta^i_{s_i(K+1),s_i(K)}$ be the smallest magnitude of a suboptimal gap on any one of the stationary segments.
 The regret of \textup{\texttt{GLRT-CascadeUCB}} is established by choosing  $\delta=\frac{1}{T}$ and $p=\sqrt{\frac{NL\log{T}}{T}}$:
    \begin{equation}
        \label{eq:u1}
        \mathcal{R}(T)=\mathcal{O}\left(\frac{N(L-K)\log{T}}{\Delta^{\textup{min}}_{\textup{opt}}}+\frac{\sqrt{NLT\log{T}}}{\left(\Delta^{\textup{min}}_{\textup{change}}\right)^2}\right).
    \end{equation}
\end{corollary}
\begin{proof}
Please refer to Appendix~\ref{sec:proof_col1} for proof.
\end{proof}

As a direct result of Theorem~\ref{thm:1}, the upper bound on the regret of \texttt{GLRT-CascadeUCB} in Corollary~\ref{col:1} consists of two terms, where the first term is incurred by the UCB-based exploration and the second term is from the change-point detection component. As $T$ becomes larger, the regret is dominated by the cost of the change-point detection component, implying the regret is $\mathcal{O}(\sqrt{NLT\log{T}}/(\Delta^{\textup{min}}_{\textup{change}})^2)$. Similar phenomena can also be found in piecewise-stationary MAB~\citep{liu2018change,cao2019nearly,besson2019generalized}. 

The proof outline of Theorem~\ref{thm:1} is as following. We can decompose $\mathcal{R}(T)$ into good events that \texttt{GLRT-CascadeUCB} reinitializes the algorithm correctly and quickly after all change-points and bad events that either large detection delays or false alarms happen. We first upper bound the regret of the stationary scenario and the detection delays of good events, respectively. It can be shown that with high probability, all change-points can be detected correctly and quickly, so that the regret incurred by bad events is rather small. By summing up all regrets from good events and bad events, an upper bound on the regret of \texttt{GLRT-CascadeUCB} is then developed. 

\subsection{Regret Upper Bound for \texttt{GLRT}-\texttt{CascadeKL}-\texttt{UCB}}

\label{sec:5.2}
This subsection deals with the upper bound on the $T$-step regret of \texttt{GLRT-CascadeKL-UCB}. 
\begin{theorem}
\label{thm:2}
Suppose that Assumptions~\ref{assump:ind} and~\ref{assump:segment} are satisfied, \textup{\texttt{GLRT-CascadeKL-UCB}} guarantees
\begin{align*}
    \mR(T)&\le \underbrace{T(N-1)(L+1)\delta}_{(a)}+\underbrace{Tp}_{(b)}+\underbrace{\sum_{i=1}^{N-1}d_i}_{(c)}+\underbrace{NK\log{\log{T}}+\sum_{i=0}^{N-1}\widetilde{D}_i}_{(d)},
\end{align*}
where $\widetilde{D}_i$ is a term depending on $\log{T}$ and the suboptimal gaps. Detailed expression can be found in~\eqref{eq:D_i} in the Appendix~\ref{sec:proof_thm_2}.
\end{theorem}
\begin{proof}
Please refer to Appendix~\ref{sec:proof_thm_2} for proof.
\end{proof}

Similarly, the upper bound on the regret of \texttt{GLRT-CascadeKL-UCB} in Theorem~\ref{thm:2} can be decomposed into four different terms, where (a) is incurred by the incorrect change-point detections, (b) is the cost of the uniform exploration, (c) is incurred by the change-point detection delay, and (d) is the cost of the KL-UCB based exploration. 

\begin{corollary}
\label{col:2}
 Choosing the same $\delta$ and $p$ as in Corollary~\ref{col:1}, \textup{\texttt{GLRT-CascadeKL-UCB}} has same order of regret upper bound as~\eqref{eq:u1}. 
\end{corollary}
\begin{proof}
The proof is similar to that of Corollary~\ref{col:1}. 
\end{proof}

We sketch the proof for Theorem~\ref{thm:2} as follows, and the detailed proofs are presented in Appendix~\ref{sec:proof_KLUCB}. By defining the events $\mathcal{U}$ and $\mathcal{H}_T$ as the algorithm performing uniform exploration and the change-points can be detected correctly and quickly, we can first bound the cost of uniform exploration $\mathcal{U}$ and cost of incorrect and slow detection of change-points $\overline{\mathcal{H}}_T$. Then, we can divide the regret $\mR(T)$ into different piecewise-stationary segments. By bounding the cost of detection delays and the KL-UCB based exploration, the upper bound on regret is thus established.

\subsection{Minimax Regret Lower Bound}
\label{subsec:lb}
In this subsection, we derive a minimax regret lower bound for piecewise-stationary CB, which is tighter than $\Omega(\sqrt{T})$ proved in~\citet{li2019cascading}. The proof technique is significantly different from~\citet{li2019cascading}.   

\begin{theorem}
\label{thm:lb}
If $L\geq 3$ and $T\geq MN\frac{(L-1)^2}{L}$, then for any policy, the worst-case regret is at least $\Omega(\sqrt{NLT})$,
where $M=1/\log \tfrac{4}{3}$, and $\Omega(\cdot)$ notation hides a constant factor that is independent of $N$, $L$, and $T$.
\end{theorem}
\begin{proof}
Please refer to Appendix~\ref{sec: sup_pf_lb} for details.
\end{proof}

The high-level idea is constructing a randomized hard instance appropriate for the piecewise-stationary CB setting, in which per time slot there is only one item with highest click probability and the click probabilities of remaining items are the same. When the distribution change occurs, the best item changes uniformly at random. For this instance, in order to lower bound the regret, it suffices to upper bound the expected numbers of appearances of the optimal item in the list. We then apply a change of measure technique to upper bound this expectation. One key step is to apply the data processing inequality for KL divergence to upper bound the discrepancy of feedback $F_t$ under change of distribution. 

This lower bound is the first characterization involving $N$, $L$, and $T$. And it indicates our proposed algorithms are nearly order-optimal within a logarithm factor $\sqrt{\log T}$.

\subsection{Discussion}

\label{sec:5.3}
 Corollaries~\ref{col:1} and~\ref{col:2} reveal that by properly choosing the confidence level $\delta$ and the uniform exploration probability $p$, the regrets of \texttt{GLRT-CascadeUCB} and \texttt{GLRT-CascadeKL-UCB} can be upper bounded by 
\begin{equation*}
    \mathcal{R}(T)=\mathcal{O}\left(\sqrt{NLT\log{T}}\right),
\end{equation*}
where $\mathcal{O}(\cdot)$ notation hides the gap term $\Delta_{\textup{change}}^\textup{min}$ and the lower order term $N(L-K)\log{T}/\Delta^{\textup{min}}_{\textup{opt}}$. Note that compared to CUSUM in ~\citet{liu2018change} and CMSW in~\citet{cao2019nearly}, the tuning parameters are fewer and does not require the smallest magnitude among the change-points $\Delta_{\textup{change}}^\textup{min}$ as shown in Corollary~\ref{col:1}. Moreover, parameter $\delta$ and $p$ follow simple rules as shown in Corollary~\ref{col:1}, while complicated parameter tuning steps are required in CUSUM and CMSW. 

The upper bounds on the regret of \texttt{GLRT-CascadeUCB} and \texttt{GLRT-CascadeKL-UCB} are improved over state-of-the-art algorithms \texttt{CascadeDUCB} and \texttt{CascadeSWUCB} in~\citet{li2019cascading} either in the dependence on $L$ or both $L$ and $T$, as their upper bounds are $\mathcal{O}(L\sqrt{NT}\log{T})$ and $\mathcal{O}(L\sqrt{NT\log{T}})$, respectively. In real-world applications, both $L$ and $T$ can be huge, for example, $L$ and $T$ are in the millions in web search, which reveals the significance of the improved $L$ dependence in our bounds. Compared to recent works on piecewise-stationary MAB~\citep{besson2019generalized} and combinatorial MAB (CMAB)~\citep{zhou2020near} that adopt GLRT as the change-point detector, the problem setting considered herein is different. In MAB, only one selected item rather than a list of items is allowed per time slot. Notice that although CMAB~\citep{combes2015combinatorial,cesa2012combinatorial,chen2016combinatorial,wang2017improving} or non-stationary CMAB~\citep{zhou2020near,chen2020combinatorial} also allow a list of items, they have full feedback on all $K$ items under semi-bandit setting.

%% file: Experiments.tex
\section{Experiments}
\label{sec:Experiments} 
In this section, numerical experiments on both synthetic and real-world datasets are carried out to validate the effectiveness of proposed algorithms. Four baseline algorithms are chosen for comparison, where \texttt{CascadeUCB1}~\citep{kveton2015cascading} and \texttt{CascadeKL-UCB}~\citep{kveton2015cascading} are nearly optimal algorithms to handle stationary CB; while \texttt{CascadeDUCB}~\citep{li2019cascading} and \texttt{CascadeSWUCB}~\citep{li2019cascading} cope with piecewise-stationary CB through a passively adaptive manner. In addition, two oracle algorithms, \texttt{Oracle-CascadeUCB1} and \texttt{Oracle-CascadeKL-UCB}, that have access to change-point times are also selected for comparison. In particular, the oracle algorithms restart when a change-point occur. Based on the theoretical analysis by~\citet{li2019cascading}, we choose $\xi=0.5$, $\gamma=1-0.25/\sqrt{T}$  for \texttt{CascadeDUCB} and choose $\tau = 2\sqrt{T\log{T}}$ for \texttt{CascadeSWUCB}. For \texttt{GLRT-CascadeUCB} and \texttt{GLRT-CascadeKL-UCB}, we set $\delta = 1/T$ and $p=0.1\sqrt{N\log{T}/T}$.

\subsection{Synthetic Dataset}

In this experiment, let $L = 10$ and $K = 3$.
We consider a simulated piecewise-stationary environment setup as follows: i) the expected attractions of the top $K$ items remain constant over the whole time horizon; ii) in each even piecewise-stationary segment, three suboptimal items are chosen randomly and their expected attractions are set to be $0.9$; iii) in each odd piecewise-stationary segment, we reset the expected attractions to the initial state. In this experiment, we set the length of each piecewise-stationary segment to be $2500$ and choose $N=10$, which is a total of $25000$ steps. Figure~\ref{fig:exp1_tra} is a detailed depiction of the piecewise-stationary environment.

\begin{figure}[htb]
	\begin{center}
	\includegraphics[width=0.6 \linewidth]{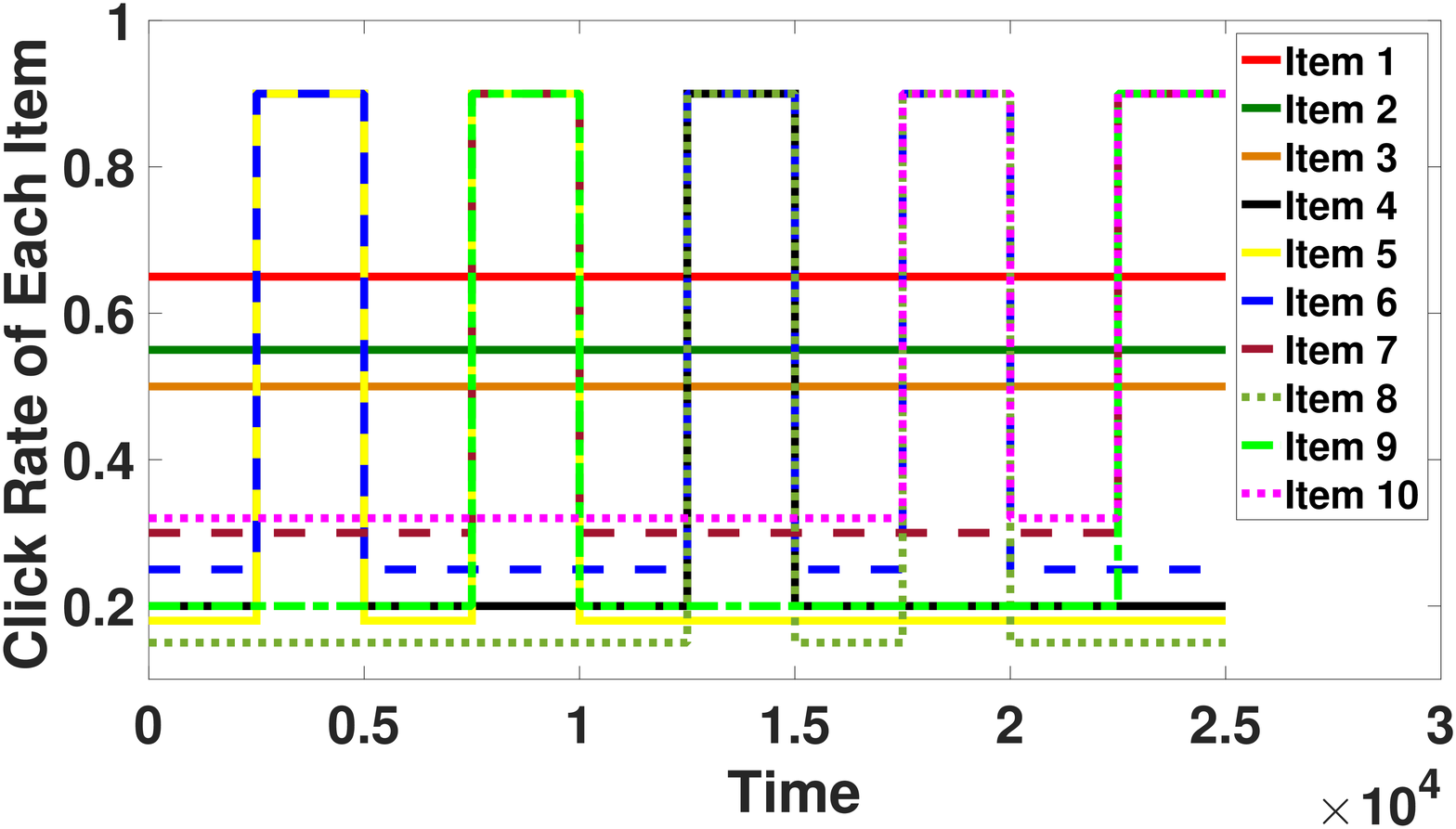}
	\end{center}
	\caption{ Click rate of each item of synthetic dataset with $T=25000$, $L=10$ and $N=10$.}
	\label{fig:exp1_tra}
\end{figure}

\begin{figure}[htb]
	\begin{center}
	\includegraphics[width=0.6 \linewidth]{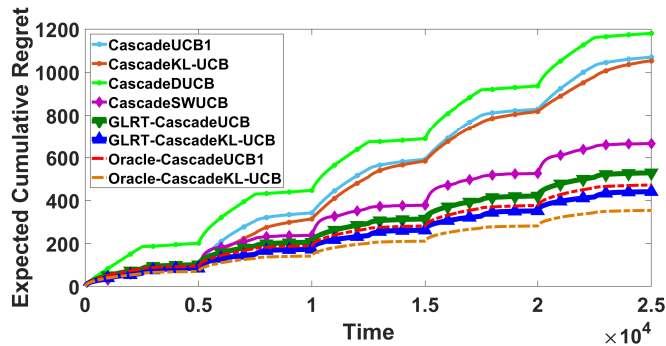}
	\end{center}
	\caption{Expected cumulative regrets of different algorithms on synthetic dataset.}
	\label{fig:exp1}
\end{figure}

Figure~\ref{fig:exp1} report the $T$-step cumulative regrets of all the algorithms by taking the average of the regrets over $100$ Monte Carlo simulations. Meanwhile, Table~\ref{table:std} lists the means and standard deviations of the $T$-step regrets of all algorithms on synthetic dataset  . The results show that the proposed \texttt{GLRT-CascadeUCB} and \texttt{GLRT-CascadeKL-UCB} achieve better performances than other algorithms and are very close to the oracle algorithms. Compared with the best existing algorithm (\texttt{CascadeSWUCB}), \texttt{GLRT-CascadeUCB} achieves a $20\%$ reduction of the cumulative regret and this fraction is $33\%$ for \texttt{GLRT-CascadeKL-UCB}, which is consistent with difference of empirical results between passively adaptive approach and actively adaptive approach in MAB. Notice that although \texttt{CascadeDUCB} seems to capture the change-points, the performance is even worse than algorithms designed for stationary CB. TThere are two possible reasons: i) The theoretical result shows that \texttt{CascadeDUCB} is worse than other algorithms for piecewise-stationary CB by a $\sqrt{\log{T}}$ factor; ii) the time horizon $\mT$ is not long enough. It is worth mentioning that our experiment on this synthetic dataset violates Assumption~\ref{assump:segment}, as it would require more than $10^5$ time slots for each piecewise-stationary segment. Surprisingly, the proposed algorithms are capable of detecting all the change-points correctly with high probability and sufficiently fast in our experiments, as shown in Table~\ref{tab:delay1}.

\subsection{Yahoo! Dataset}

\begin{figure}[htb]
	\begin{center}
	\includegraphics[width=0.6 \linewidth]{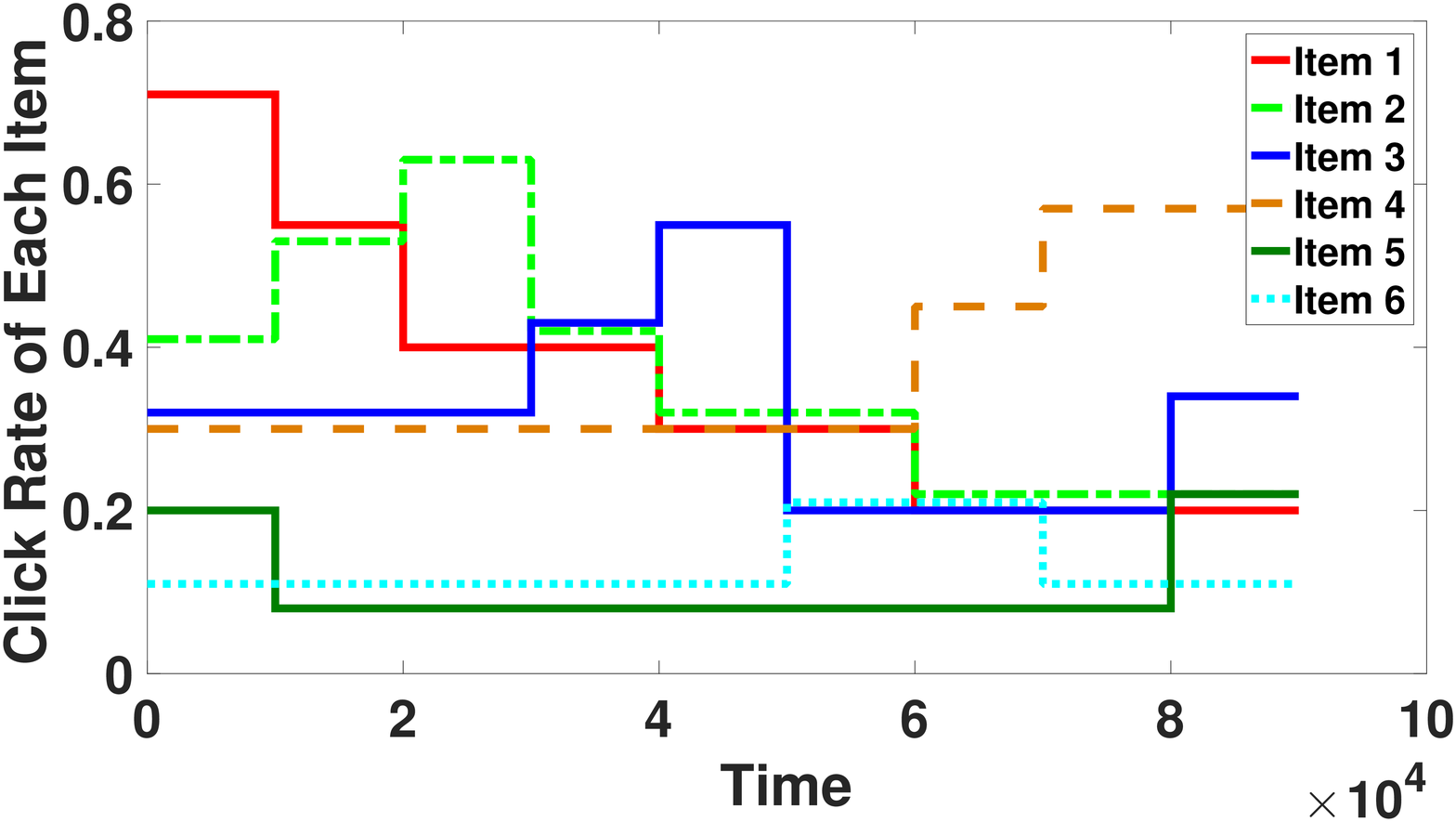}
	\end{center}
	\caption{ Click rate of each item of Yahoo! dataset with $T=90000$, $L=6$ and $N=9$.}
	\label{fig:exp2_tra}
\end{figure}

\begin{figure}[htb]
	\begin{center}
	\includegraphics[width=0.6 \linewidth]{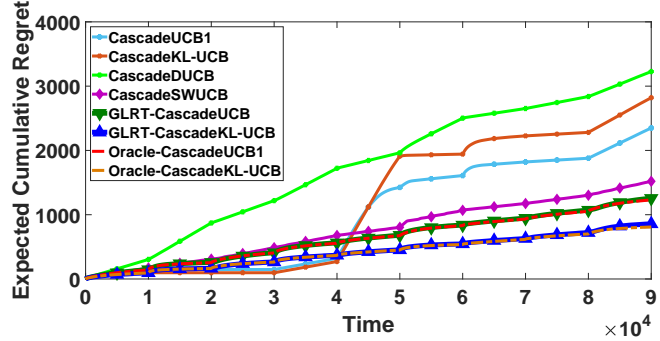}
	\end{center}
	\caption{Expected cumulative regrets of different algorithms on Yahoo! dataset.}
	\label{fig:exp2}
\end{figure}

In this subsection, we adopt the benchmark dataset for the evaluation of bandit algorithms published by Yahoo!\footnote{Yahoo! Front Page Today Module User Click Log Dataset on https://webscope.sandbox.yahoo.com}. This dataset, using binary values to indicate if there is a click or not, contains user click log for news articles displayed in the Featured Tab of the Today Module on Yahoo!~\citep{li2011unbiased}, where each item corresponds to one article. We pre-process the dataset by adopting the same method as~\citet{cao2019nearly}, where $L=6$, $K=2$ and $N=9$. To make the experiment nontrivial, several modifications are applied to the dataset: i) the click rate of each item is enlarged by $10$ times; ii) the time horizon is reduced to $T=90000$, which is shown in Figure~\ref{fig:exp2_tra}. Figure~\ref{fig:exp2} presents the cumulative regrets of all algorithms by averaging $100$ Monte Carlo trials, which shows  the regrets of our proposed algorithms are just slightly above the oracle algorithms and significantly outperform other algorithms. The reason that algorithms designed for stationarity perform better in the first $3$ segments is the optimal list does not change. Table~\ref{table:std} lists the means and standard deviations of the $T$-step regrets of all algorithms on Yahoo! dataset. Again, although the Assumption~\ref{assump:segment} is not satisfied in the Yahoo! dataset, GLRT based algorithms detect the change-points correctly and quickly and detailed mean detection time of each change-point with its standard deviation is in Table~\ref{tab:delay2}.

\begin{table*}[htb]
	\footnotesize
	\caption{Means and standard deviations of the $T$-step regrets.}
	\centering
	\begin{tabular}{|c||c|c|c|c|}
		\hline\hline
		& \texttt{CascadeUCB1}& \texttt{CascadeKL-UCB}& \texttt{CascadeDUCB}  & \texttt{CascadeSWUCB} \\ \hline
		Synthetic Dataset & $1069.77\pm 87.09$ & $1053.25\pm 111.67$ & $1180.30\pm 20.22$ & $664.84\pm 29.81$       \\ \hline
		Yahoo! Experiment & $2349.29\pm 312.71$ & $2820.16\pm 256.74$ & $3226.97\pm 39.37$ & $1519.56\pm 52.23$ \\
		\hline\hline
		& \texttt{GLRT-CascadeUCB}& \texttt{GLRT-CascadeKL-UCB}&\texttt{Oracle-CascadeUCB1}& \texttt{Oracle-CascadeKL-UCB}  \\
		\hline
		Synthetic Dataset & $\mathbf{527.93\pm 25.20}$ & $\mathbf{440.93\pm 45.54}$ & $472.25\pm 17.65$ & $353.86\pm 19.59$      \\ \hline
		Yahoo! Experiment & $\mathbf{1235.21\pm 54.59}$ & $\mathbf{856.77\pm 67.16}$ & $1230.17\pm 45.24$ & $808.84\pm 47.97$  \\
		\hline
	\end{tabular}
	\label{table:std}
\end{table*}
\begin{table*}[htbp!]
\footnotesize
\caption{Means and standard deviations of detection time $\tau_i$'s of change-points for synthetic dataset}
\centering
\begin{tabular}{|c||c|c|c|c|c|}
\hline\hline
Change-points & 2500 & 5000 & 7500 & 10000 & 12500       \\ 
\hline
\texttt{GLRT-CascadeUCB} & $2617.34\pm 102.49$ & $5022.83\pm 7.73$ & $7630.67\pm 124.99$ & $10025.62\pm 9.25$ & $12623.72\pm 110.34$ \\
\hline
\texttt{GLRT-CascadeKL-UCB} & $2722.82\pm 190.25$ & $5019.68\pm 6.88$ & $7721.75\pm 202.44$ & $10019.29\pm 5.76$ & $12707.02\pm 160.31$ \\
\hline\hline
Change-points & 15000 & 17500 & 20000 & 22500 & \\
\hline
\texttt{GLRT-CascadeUCB} & $15023.16\pm 7.83$ & $17614.26\pm 77.37$ & $20024.80\pm 7.66$ & $22607.32\pm 78.42$ &\\
\hline
\texttt{GLRT-CascadeKL-UCB} & $15019.29\pm 6.50$ & $17737.52\pm 173.15$ & $20021.05\pm 7.36$ & $22675.48\pm 157.98$ &\\
\hline
\end{tabular}
\label{tab:delay1}
\end{table*}

\begin{table*}[htb]
\footnotesize
\caption{Means and standard deviations of detection time $\tau_i$'s of change-points for Yahoo! dataset}
\centering
\begin{tabular}{|c||c|c|c|c|}
\hline\hline
Change-points & 10000 & 20000 & 30000 & 40000       \\ 
\hline
\texttt{GLRT-CascadeUCB} & $10378.89\pm 122.84$ & $20803.91\pm 387.20$ & $30249.65\pm 74.82$ & $40771.35\pm 257.24$  \\
\hline
\texttt{GLRT-CascadeKL-UCB} & $10362.85\pm 118.28$ & $20673.78\pm 338.82$ & $30224.11\pm 68.01$ & $40781.40\pm 253.39$  \\
\hline\hline
Change-points & 50000 & 60000 & 70000 & 80000  \\
\hline
\texttt{GLRT-CascadeUCB} & $50084.44\pm 20.82$ & $60474.27\pm 163.44$ & $70768.15\pm 244.04$ & $81490.03\pm 804.54$ \\
\hline
\texttt{GLRT-CascadeKL-UCB} & $50084.81\pm 22.44$ & $60484.09\pm 159.79$ & $70766.37\pm 246.76$ & $81785.40\pm 1285.24$ \\
\hline
\end{tabular}
\label{tab:delay2}
\end{table*}

%% file: Conclusion.tex
\section{Conclusion}
\label{sec:Conclusion} 
Two new actively adaptive algorithms for piecewise-stationary cascading bandit, namely \texttt{GLRT-CascadeUCB} and \texttt{GLRT-CascadeKL-UCB} are developed in this work. It is analytically established that \texttt{GLRT-CascadeUCB} and \texttt{GLRT-CascadeKL-UCB} achieve the same nearly optimal regret upper bound on the order of $\mathcal{O}\left(\sqrt{NLT\log{T}}\right)$, which matches our minimax regret lower bound up to a $\sqrt{\log T}$ factor. Compared with state-of-the-art algorithms that adopt passively adaptive approach such as \texttt{CascadeSWUCB} and \texttt{CascadeDUCB}, our new regret upper bounds are reduced by $\mathcal{O}(\sqrt{L})$ and $\mathcal{O}(\sqrt{L\log{T}})$ respectively. Numerical tests on both synthetic and real-world data show the improved efficiency of the proposed algorithms. Several interesting questions are still left open for future work. One challenging problem lies in whether the $\sqrt{\log{T}}$ gap in time steps $T$ between regret upper bound and lower bound  can be closed. In addition, we are also interested in extending the single click models to multiple clicks models in future work.  

%% file: Appendix.tex
\newpage
\onecolumn
\appendix
\begin{center}
\textbf{\Large Appendices}
\end{center}

\section{Detailed Proofs of Theorem~\ref{thm:1}}
\label{sec:proof_UCB}
\subsection{Proofs of Auxiliary Lemmas}
\label{sec:proof_lemma}
In this subsection, we present auxiliary lemmas which are used to prove Theorem~\ref{thm:1}, as well as their proofs.\\
We start by upper bounding the regret under the stationary scenario with $N=1$, $\nu_0=0$, and $\nu_1=T$.
\begin{lemma}
\label{lem:stationary_bandit}
Under stationary scenario ($N=1$), the regret of \textup{\texttt{GLRT-CascadeUCB}} is upper bounded as
\begin{equation*}
    \mathcal{R}(T)\le T\mathbb{P}\left(\tau_1\le T\right)+pT+\widetilde{C}_1,
\end{equation*}
where $\tau_1$ is the first detection time.
\end{lemma}
\begin{proof}[Proof of Lemma \ref{lem:stationary_bandit}]
Denote as $R_t:=R\left(\mathcal{A}_t,\mathbf{w}_t,\mathbf{Z}_t\right)$ the regret of the learning algorithm at time slot $t$, where $\mathcal{A}_t$ is the recommended list at time slot $t$ and $\mathbf{w}_t$ is the associated expected attraction vector at time slot $t$. By further denoting as $\tau_1$ the first change-point detection time of the Bernoulli GLRT, the regret of \texttt{GLRT-CascadeUCB} can be decomposed as:
\begin{align*}
    \mathcal{R}(T)=\mathbb{E}\left[\sum_{t=1}^T R_t\mathbb{I}\{\tau_1\le T\}\right]+\mathbb{E}\left[\sum_{t=1}^T R_t\mathbb{I}\{\tau_1> T\}\right]
    \overset{(a)}{\le} T\mathbb{P}\left(\tau_1\le T\right)+\underbrace{\mathbb{E}\left[\sum_{t=1}^T R_t\mathbb{I}\{\tau_1> T\}\right]}_{(b)},
\end{align*}
where inequality (a) holds due to the fact that $R_t\le 1$ and $\mathbb{E}\left[\mathbb{I}\{\tau_1\le T\}\right]=\mathbb{P}\left(\tau_1\le T\right)$. 

In order to bound the term (b), we denote the event $\mathcal{U}$ as the algorithm being in the forced uniform exploration phase and let $\mathcal{E}_t:=\{\exists\ell\in\mL~\mbox{s.t.}~|\mathbf{w}^1(\ell)-\hat{\mathbf{w}}_t(\ell)|\ge\sqrt{3\log{t}/(2n_{\ell,t})}\}$ be the event that $\hat{\mathbf{w}}_t(\ell)$ is not in the high-probability confidence interval around $\mathbf{w}^1(\ell)$, where $\mathbf{w}^1(\ell)$ is expected attraction of item $\ell$ in the first piecewise-stationary segment, $\hat{\mathbf{w}}_t(\ell)$ is the sample mean of item $\ell$ up to time slot $t$, and $n_{\ell,t}$ is the number of times that item $\ell$ is observed up to time slot $t$. Term (b) can be further decomposed as
\begin{align*}
    \mathbb{E}\left[\sum_{t=1}^T R_t\mathbb{I}\{\tau_1> T\}\right]&=\mathbb{E}\left[\sum_{t=1}^TR_t\mathbb{I}\{\mathcal{U}\}\right]+\mathbb{E}\left[\sum_{t=1}^T R_t\mathbb{I}\{\tau_1> T,\mathcal{E}_{t-1},\overline{\mathcal{U}}\}\right]+\mathbb{E}\left[\sum_{t=1}^T R_t\mathbb{I}\{\tau_1> T,\overline{\mathcal{E}}_{t-1},\overline{\mathcal{U}}\}\right]\\
    &\overset{(c)}{\le} Tp+\underbrace{\mathbb{E}\left[\sum_{t=1}^T R_t\mathbb{I}\{\tau_1> T,\mathcal{E}_{t-1},\overline{\mathcal{U}}\}\right]}_{(d)}+\underbrace{\mathbb{E}\left[\sum_{t=1}^T R_t\mathbb{I}\{\tau_1> T,\overline{\mathcal{E}}_{t-1},\overline{\mathcal{U}}\}\right]}_{(e)},
\end{align*}
where inequality (c) is because of the fact that $R_t\le 1$ and the uniform exploration probability is $p$. Term (d) can be bounded by applying the Chernoff-Hoeffding inequality,
\begin{align*}
   \mathbb{E}\left[\sum_{t=1}^T R_t\mathbb{I}\{\tau_1> T,\mathcal{E}_{t-1},\overline{\mathcal{U}}\}\right] \le\sum_{\ell=1}^L\sum_{t=1}^T\sum_{n_\ell=1}^t\Prob\left(|\mathbf{w}^1(\ell)-\hat{\mathbf{w}}_t(\ell)|\ge\sqrt{3\log{t}/(2n_{\ell})}\right)
   \le2\sum_{\ell=1}^L\sum_{t=1}^T\sum_{n_\ell=1}^te^{-3\log{t}}\le\frac{\pi^2L}{3}.
\end{align*}
Furthermore, term (e) can be bounded as follows,
\begin{align*}
    \mathbb{E}\left[\sum_{t=1}^T R_t\mathbb{I}\{\tau_1> T,\overline{\mathcal{E}}_{t-1},\overline{\mathcal{U}}\}\right]&\overset{(f)}{\le}pT+\sum_{\ell=K+1}^L\frac{12}{\Delta^1_{s_1(\ell),s_1(K)}}\log{T},
\end{align*}
where the inequality $(f)$ follows the proof of Theorem 2 in~\citet{kveton2015cascading}. By summing all terms, we prove the result.
\end{proof}
Then we bound the false alarm probability $\mathbb{P}\left(\tau_1\le T\right)$ in Lemma~\ref{lem:stationary_bandit} under previously mentioned stationary scenario.
\begin{lemma}
\label{lem:false_alarm}
Consider the stationary scenario, with confidence level $\delta\in (0,1)$ for the Bernoulli GLRT, and we have that
\begin{equation*}
\mathbb{P}\left(\tau_1\le T\right)\le L\delta.   
\end{equation*}
\end{lemma}
\begin{proof}[Proof of Lemma \ref{lem:false_alarm}]
Define $\tau_{\ell,1}$ as the first change-point detection time of the $\ell$th item. Then, $\tau_1=\min_{\ell\in\mathcal{L}}\tau_{\ell,1}$. Since the global restart is adopted, by applying the union bound, we have that
\begin{equation*}
    \mathbb{P}\left(\tau_1\le T\right)\le\sum_{\ell=1}^L\mathbb{P}\left(\tau_{\ell,1}\le T\right).
\end{equation*}
Recall the GLR statistic defined in~\eqref{eq:GLR_STAT}, and plug it into $\mathbb{P}\left(\tau_{\ell,1}\le T\right)$, we have that
\begin{align*}
    \mathbb{P}\left(\tau_{\ell,1}\le\tau\right)&\le\mathbb{P}\left[\exists (s,n)\in\mathbb{N}^2,n\le n_\ell, s<n:s\mbox{KL}\left(\hat{\mu}^1_{\ell,1:s},\hat{\mu}^1_{\ell,1:n}\right)+(n-s)\mbox{KL}\left(\hat{\mu}^1_{\ell,s+1:n},\hat{\mu}^1_{\ell,1:n}\right)>\beta(n,\delta)\right]\\
    &\le\mathbb{P}\left[\exists (s,n)\in\mathbb{N}^2,n\le T, s<n:s\mbox{KL}\left(\hat{\mu}^1_{\ell,1:s},\hat{\mu}^1_{\ell,1:n}\right)+(n-s)\mbox{KL}\left(\hat{\mu}^1_{\ell,s+1:n},\hat{\mu}^1_{\ell,1:n}\right)>\beta(n,\delta)\right]\\
    &\overset{(a)}{\le}\mathbb{P}\left[\exists (s,n)\in\mathbb{N}^2,n\le T, s<n:s\mbox{KL}\left(\hat{\mu}^1_{\ell,1:s},\mathbf{w}^1(\ell)\right)+(n-s)\mbox{KL}\left(\hat{\mu}^1_{\ell,s+1:n},\mathbf{w}^1(\ell)\right)>\beta(n,\delta)\right]\\
    &\overset{(b)}{\le}\sum_{s=1}^T\mathbb{P}\left[\exists s<n:s\mbox{KL}\left(\hat{\mu}^1_{\ell,1:s},\mathbf{w}^1(\ell)\right)+(n-s)\mbox{KL}\left(\hat{\mu}^1_{\ell,s+1:n},\mathbf{w}^1(\ell)\right)>\beta(n,\delta)\right]\\
    &\le\sum_{s=1}^T\mathbb{P}\left[\exists r\in\mathbb{N}:s\mbox{KL}\left(\hat{\mu}^1_{\ell,s},\mathbf{w}^1(\ell)\right)+r\mbox{KL}\left(\hat{\mu}^1_{k,r},\mathbf{w}^1(\ell)\right)>\beta(s+r,\delta)\right]\\
    &\overset{(c)}{\le}\sum_{s=1}^T\frac{\delta}{3s^{3/2}}
    \overset{(d)}{\le}\sum_{s=1}^\infty\frac{\delta}{3s^{3/2}}
    \le\delta,
\end{align*}
where $\hat{\mu}_{\ell,s:s'}^1$ is the mean of the rewards generated from the distribution $f^1_{\ell}$ with expected reward $\mathbf{w}^1(\ell)$ from time slot $s$ to $s'$. Inequality (a) is because of the fact that
\begin{equation*}
    s\mbox{KL}\left(\hat{\mu}_{1:s},\hat{\mu}_{1:n}\right)+(n-s)\mbox{KL}\left(\hat{\mu}_{s+1:n},\hat{\mu}_{1:n}\right)=\inf_{\lambda\in[0,1]}\left[s\mbox{KL}\left(\hat{\mu}_{1:s},\lambda\right)+(n-s)\mbox{KL}\left(\hat{\mu}_{s+1:n},\lambda\right)\right];
\end{equation*}
inequality (b) is because of the union bound; inequality (c) is because of the Lemma 10 in~\citet{besson2019generalized}; and inequality (d) holds due to the Riemann zeta function $\zeta(x)$ and when $x=3/2$, $\zeta(3/2)< 2.7$. Thus, we conclude by $\mathbb{P}\left(\tau_1\le T\right)\le L\delta$.
\end{proof}
Next, we define the event $\mathcal{C}^{(i)}$ that all the change-points up to $i$th have been detected quickly and correctly:
\begin{equation}
\label{eq:C_event}
    \mathcal{C}^{(i)}=\left\{\forall j\le i,\tau_j\in\left\{\nu_j+1,\cdots,\nu_j+d_j\right\}\right\}.
\end{equation}
Lemma~\ref{lem:cond_prob} below shows $\mathcal{C}^{(i)}$ happens with high probability.
\begin{lemma}
\label{lem:cond_prob}
(Lemma 12 in~\citet{besson2019generalized}) When $\mathcal{C}^{(i-1)}$ holds, \textup{GLRT} with confidence level $\delta$ is capable of detecting the change point $\nu_i$ correctly and quickly with high probability, that is,
\begin{align*}
&\mathbb{P}\left[\tau_i\le\nu_i|\mC^{(i-1)}\right]\le L\delta,~\mbox{and}~\mathbb{P}\left[\tau_i\ge\nu_i+d_i|\mC^{(i-1)}\right]\le \delta,
\end{align*}
where $\tau_i$ is the detection time of $i$th change-point.
\end{lemma}

In the next lemma, we bound the expected detection delay with the good event $\mC^{(i)}$ holds.
\begin{lemma}
\label{lem:bound_delay}
The expected delay given $\mC^{(i)}$ is: 
\begin{equation*}
\mathbb{E}\left[\tau_i-\nu_i|\mathcal{C}^{(i)}\right]\le d_i.
\end{equation*}
\end{lemma}
\begin{proof}
By the definition of $\mathcal{C}^{(i)}$, the conditional expected delay is obviously upper bounded by $d_i$.
\end{proof}
\subsection{Proof of Theorem~\ref{thm:1}}
\label{sec:proof_thm1}
\begin{proof}
Define good events $E_i=\{\tau_i>\nu_i\}$ and $D_i=\{\tau_i\le\nu_i+d_i\}$, $\forall 1\le i\le N-1$. Recall the definition of the good event $\mC^{(i)}$ that all the change-points up to $i$th one have been detected correctly and quickly in~\eqref{eq:C_event}, and we can find that $\mathcal{C}^{(i)}=E_1\cap D_1\cap\cdots \cap E_i\cap D_i$. Again, we denote $R_t:=R\left(\mathcal{A}_t,\mathbf{w}_t,\mathbf{Z}_t\right)$ as the regret of the learning algorithm at time slot $t$. By first decomposing the expected cumulative regret with respect to the event $E_1$, we have that
\begin{align*}
    \mathcal{R}(T)&=\E\left[\sum_{t=1}^T R_t\mathbb{I}{\{E_1\}}\right]+\E\left[\sum_{t=1}^T R_t\mathbb{I}{\{\overline{E_1}\}}\right]\le\E\left[\sum_{t=1}^T R_t\mathbb{I}{\{E_1\}}\right]+T\mathbb{P}(\overline{E}_1)\\
    &\overset{(a)}{\le}\E\left[\sum_{t=1}^{\nu_1}R_t\mathbb{I}{\{E_1\}}\right]+\E\left[\sum_{t=\nu_1+1}^TR_t\right]+TL\delta\overset{(b)}{\le}\widetilde{C}_1+\nu_1 p+\underbrace{\E\left[\sum_{t=\nu_1+1}^TR_t\right]}_{(c)}+TL\delta,
\end{align*}
where the inequality (a) is because that $\mathbb{P}(\overline{E}_1)$ can be bounded using Lemma~\ref{lem:false_alarm} and inequality (b) holds due to Lemma~\ref{lem:stationary_bandit}.
To bound the term (c), by applying the law of total expectation, we have that 
\begin{align*}
    \E\left[\sum_{t=\nu_1+1}^TR_t\right]&\le\E\left[\sum_{t=\nu_1+1}^TR_t\bigm|\mathcal{C}^{(1)}\right]+T(1-\mathbb{P}(E_1\cap D_1))=\E\left[\sum_{t=\nu_1+1}^TR_t\bigm|\mathcal{C}^{(1)}\right]+T(\mathbb{P}(\overline{E}_1\cup \overline{D}_1))\\
    &\le\underbrace{\E\left[\sum_{t=\nu_1+1}^TR_t\bigm|E_1\cap D_1\right]}_{(d)}+T(L+1)\delta,
\end{align*}
where $\mathbb{P}(\overline{E}_1\cup \overline{D}_1)$ is acquired by applying the union bound on the Lemma~\ref{lem:cond_prob}.
Then, we turn to the term (d), by further splitting the regret, 
\begin{align*}
    \E\left[\sum_{t=\nu_1+1}^TR_t\bigm|E_1\cap D_1\right]=\E\left[\sum_{t=\nu_1+1}^TR_t\bigm|\mathcal{C}^{(1)}\right]
    &\le\E\left[\sum_{t=\tau_1+1}^TR_t\bigm|\mathcal{C}^{(1)}\right]+\underbrace{\E\left[\sum_{t=\nu_1+1}^{\tau_1} R_t\bigm|\mC^{(1)}\right]}_{(e)}
    \\
    &\le\E\left[\sum_{t=\nu_1+1}^TR_t\bigm|\mathcal{C}^{(1)}\right]+d_1,
\end{align*}
where term (e) is bounded by applying the Lemma~\ref{lem:bound_delay} and the fact that $R_t\le 1$.
Thus,
\begin{equation*}
 \mathcal{R}(T)\le\E\left[\sum_{t=\nu_1+1}^TR_t\bigm|\mathcal{C}^{(1)}\right]+\widetilde{C}_1+\nu_1p+d_1+3TL\delta.   
\end{equation*}

Similarly,
\begin{align*}
  \E\left[\sum_{t=\nu_1+1}^TR_t\bigm|\mathcal{C}^{(1)}\right]&\le\E\left[\sum_{t=\nu_1+1}^TR_t\mathbb{I}{\{E_2\}}\bigm|\mathcal{C}^{(1)}\right]+T\mathbb{P}(\overline{E}_2|\mathcal{C}^{(1)})\\
  &\le\E\left[\sum_{t=\nu_1+1}^{\nu_2}R_t\mathbb{I}{\{E_2\}}\bigm|\mathcal{C}^{(1)}\right]+\E\left[\sum_{t=\nu_2+1}^TR_t\bigm|\mathcal{C}^{(1)}\right]+TL\delta\\
  &\le \widetilde{C}_2+(\nu_2-\nu_1)p+\underbrace{\E\left[\sum_{t=\nu_2+1}^TR_t\bigm|\mathcal{C}^{(1)}\right]}_{(f)}+TL\delta,
\end{align*}
where $\mathbb{P}(\overline{E}_2|\mathcal{C}^{(1)})$ directly follows Lemma~\ref{lem:cond_prob}. To bound term (f),
\begin{align*}
\E\left[\sum_{t=\nu_2+1}^TR_t\bigm|\mathcal{C}^{(1)}\right]&\le\E\left[\sum_{t=\nu_2+1}^TR_t\bigm|E_2\cap D_2\cap \mathcal{C}^{(1)}\right]+T(1-\mathbb{P}(E_2\cap D_2|\mathcal{C}^{(1)}))\\
&=\E\left[\sum_{t=\nu_2+1}^TR_t\bigm|E_2\cap D_2\cap \mathcal{C}^{(1)}\right]+T\mathbb{P}(\overline{E}_2\cup \overline{D}_2|\mathcal{C}^{(1)})\\
&\le \underbrace{\E\left[\sum_{t=\nu_2+1}^TR_t\bigm|\mathcal{C}^{(2)}\right]}_{(g)}+T(L+1)\delta,
\end{align*}
where $\mathbb{P}(\overline{E}_2\cup \overline{D}_2|\mathcal{C}^{(1)})$ is acquired by applying the union bound on Lemma~\ref{lem:cond_prob}.
For term (g), we have
\begin{align*}
    \E\left[\sum_{t=\nu_2+1}^TR_t\bigm|\mathcal{C}^{(2)}\right]\le\E\left[\sum_{t=\tau_2+1}^TR_t\bigm|\mathcal{C}^{(2)}\right]+\E\left[\sum_{t=\nu_2+1}^{\tau_2}R_t\bigm|\mathcal{C}^{(2)}\right]
     \le\E\left[\sum_{t=\nu_2+1}^TR_t\bigm|\mathcal{C}^{(2)}\right]+d_2.
\end{align*}
Wrapping up previous steps, we have that
\begin{equation*}
\mathcal{R}(T)\le\E\left[\sum_{t=\nu_2+1}^TR_t\bigm|\mathcal{C}^{(2)}\right]+\widetilde{C}_1+\widetilde{C}_2+\nu_2p+d_1+d_2+6TL\delta.    
\end{equation*}
 Recursively, the upper bound on the regret of \texttt{GLRT-CascadeUCB} is given by
\begin{equation*}
    \mathcal{R}(T)\le\sum_{i=1}^N \widetilde{C}_i+Tp+\sum_{i=1}^{N-1}d_i+3NTL\delta.
\end{equation*}
\end{proof}
\subsection{Proof of Corollary~\ref{col:1}}
\label{sec:proof_col1}
\begin{proof}
By applying the upper bound on $\mathcal{G}(x)$ that $\mathcal{G}(x)\le x+4\log(1+x+\sqrt{2x})$ if $x\ge 5$ to $d_i$, we have that
\begin{align*}
    d_i&\le\frac{4L}{p\left(\Delta^{\textup{min}}_{\textup{change}}\right)^2\beta(T,\delta)}+\frac{2L}{p}\\
    &\overset{(a)}{\le}\frac{4L}{p\left(\Delta^{\textup{min}}_{\textup{change}}\right)^2}\left[\log\left(\frac{3T^{3/2}}{\delta}\right)+8\log\left(1+\frac{\log(\frac{3T^{3/2}}{\delta})}{2}+\sqrt{\log\left(\frac{3T^{3/2}}{\delta}\right)}\right)+6\log(1+\log{T})\right]+\frac{2L}{p}\\
    &\overset{(b)}{\le}\frac{\frac{20L\log{T}+o(L\log{T})}{\left(\Delta^{\textup{min}}_{\textup{change}}\right)^2}+2L}{p}\lesssim\frac{L\log{T}}{p \left(\Delta^{\textup{min}}_{\textup{change}}\right)^2},
\end{align*}
where $(a) (b)$ hold when $\log(3T^{5/2})\ge 10$ (equals to $T\ge 36$). By plugging $d_i$ into Theorem~\ref{thm:1}, we have that,
\begin{equation*}
    \mathcal{R}(T)\lesssim\frac{N(L-K)\log{T}}{\Delta^{\textup{min}}_\textup{opt}}+Tp+\frac{NL\log{T}}{p \left(\Delta^{\textup{min}}_{\textup{change}}\right)^2}+3NL.
\end{equation*}
Combining the above analysis we conclude the corollary.
\end{proof}

\section{Detailed Proofs of Theorem~\ref{thm:2}}
\label{sec:proof_KLUCB}
\subsection{Proof of Theorem~\ref{thm:2}}
\label{sec:proof_thm_2}
\begin{proof}[Proof of Theorem~\ref{thm:2}]
We start by defining the good event $\mathcal{H}_T$ that all the change-points have been detected correctly and quickly,
\begin{align*}
   \mathcal{H}_T:=\{\forall i=1,\ldots,N-1,\tau_i\in\{\nu_i+1,\ldots,\nu_i+d_i\}\},
\end{align*}
And let $\mathcal{E}_{t,i}:=\{\exists\ell\in\{s_i(1),\ldots,s_i(K)\}~\mbox{s.t.}~\mathbf{w}^i(\ell)>\textup{UCB}_{\textup{KL},t}(\ell)\}$ be the event that the expected attraction of at least one optimal item is above the UCB index at time slot $t$ and $t$ is in $i$th piecewise-stationary segment, where $\textup{UCB}_{\textup{KL},t}(\ell)$ is the KL-UCB index of $\ell$ item computed at time slot $t$.
The regret of $\texttt{GLRT-CascadeKL-UCB}$ can be decomposed as
\begin{align*}
    \mR(T)&\le\mathbb{E}\left[\sum_{t=1}^T R_t\mathbb{I}\{\mathcal{U}\}\right]+\mathbb{E}\left[\sum_{t=1}^T R_t\mathbb{I}\{\overline{\mathcal{U}},\overline{\mathcal{H}}_T\}\right]+\mathbb{E}\left[\sum_{t=1}^T R_t\mathbb{I}\{\overline{\mathcal{U}},\mathcal{H}_T\}\right]\\
    & \le pT+\underbrace{T\Prob(\overline{\mathcal{H}}_T)}_{(a)}+\sum_{i=1}^{N-1}d_i+\underbrace{\mathbb{E}\left[\sum_{t=1}^{\nu_1} R_t\mathbb{I}\{\overline{\mathcal{U}},\mathcal{H}_T,\mathcal{E}_{t-1,1}\}\right]}_{(b)}+\sum_{i=1}^{N-1}\underbrace{\mathbb{E}\left[\sum_{t=\tau_i+1}^{\nu_{i+1}} R_t\mathbb{I}\{\overline{\mathcal{U}},\mathcal{H}_T,\mathcal{E}_{t-1,i+1}\}\right]}_{(c)}\\
    &\quad+\underbrace{\mathbb{E}\left[\sum_{t=1}^{\nu_1} R_t\mathbb{I}\{\overline{\mathcal{U}},\mathcal{H}_T,\overline{\mathcal{E}}_{t-1,1}\}\right]}_{(d)}+\sum_{i=1}^{N-1}\underbrace{\mathbb{E}\left[\sum_{t=\tau_i+1}^{\nu_{i+1}} R_t\mathbb{I}\{\overline{\mathcal{U}},\mathcal{H}_T,\overline{\mathcal{E}}_{t-1,i+1}\}\right]}_{(e)}.
\end{align*}

\textbf{Bound Term (a)}:
Recall the definition of $\mC^{(i)}$ and applying the union bound,
\begin{align*}
    \Prob(\overline{\mathcal{H}}_T)&\le\sum_{i=1}^{N-1}\Prob(\tau_i\notin\{\nu_i+1,\ldots,\nu_i+d_i\}|\mC^{(i-1)})\\
    &\le\sum_{i=1}^{N-1}\Prob(\tau_i\le \nu_i|\mC^{(i-1)})+\sum_{i=1}^{N-1}\Prob(\tau_i\ge \nu_i+d_i|\mC^{(i-1)})\\
    &\le(N-1)(L+1)\delta,
\end{align*}
where the last inequality is due to Lemma~\ref{lem:cond_prob}.

\textbf{Bound Terms (b) and (c)}: By plugging in the event $\mathcal{E}_{t,i}$, we have that
\begin{align*}
    \mathbb{E}\left[\sum_{t=\tau_i+1}^{\nu_{i+1}} R_t\mathbb{I}\{\overline{\mathcal{U}},\mathcal{H}_T,\mathcal{E}_{t-1,i+1}\}\right]&\le \sum_{\ell^*=1}^K\mathbb{E}\left[\mathbb{I}\{\mC^{(i)}\}\sum_{t=\tau_i+1}^{\nu_{i+1}}\mathbb{I}\{n_{s_t(\ell^*),t}\mbox{KL}(\hat{\mathbf{w}}_t(s_t(\ell^*)),\mathbf{w}_t(s_t(\ell^*)))\ge g(t-\tau_i)\}\right]\\
    &\le K\mathbb{E}\left[\sum_{t=\tau_i+1}^{\nu_{i+1}}\mathbb{I}\{n_{s_t(\ell^*),t}\mbox{KL}(\hat{\mathbf{w}}_t(s_t(\ell^*)),\mathbf{w}_t(s_t(\ell^*)))\ge g(t-\tau_i)\}|\mC^{(i)}\right]\\
   &= K\mathbb{E}\left[\sum_{t=\tau_i+1}^{\nu_{i+1}}\mathbb{I}\{n_{s_t(\ell^*),t}\mbox{KL}(\hat{\mathbf{w}}_t(s_t(\ell^*)),\mathbf{w}^i(s_t(\ell^*)))\ge g(t-\tau_i)\}|\mC^{(i)}\right]\\
   &\le K\sum_{t'=1}^{\nu_{i+1}-\tau_i}\Prob(\exists s\le t':s\mbox{KL}(\hat{\mu}_s,\mathbf{w}^i(s_t(\ell^*)))\ge g(t'))\\
   &\le K\sum_{t=1}^T\frac{1}{t\log{t}}\le K\log{\log{T}},
\end{align*}
where the first inequality is due to $\mathcal{H}_T\in\mC^{(i)}$; $\hat{\mathbf{w}}(\ell)$ is the mean of the rewards of item $\ell$ after the most recent detection time $\tau$ and up to time slot $t$; and the last inequality follows directly from Lemma 2 in~\citet{cappe2013kullback}. Note that (b) can be upper bounded similar to the procedures of bounding (c). 

\textbf{Bound Terms (d) and (e)}:
Here, according to the proof of Theorem 3 in~\citet{kveton2015cascading}, (d) and (e) can be bounded as
\begin{align*}
    &\mathbb{E}\left[\sum_{t=1}^{\nu_1} R_t\mathbb{I}\{\overline{\mathcal{U}},\mathcal{H}_T,\overline{\mathcal{E}}_{t-1,1}\}\right]~\mbox{or}~\mathbb{E}\left[\sum_{t=\tau_i+1}^{\nu_{i+1}} R_t\mathbb{I}\{\overline{\mathcal{U}},\mathcal{H}_T,\overline{\mathcal{E}}_{t-1,i+1}\}\right]\\
    &\quad\le\sum_{\ell=K+1}^L\frac{(1+\epsilon)\Delta^{i+1}_{s_{i+1}(\ell),s_{i+1}(K)}(1+\log(1/\Delta^{i+1}_{s_{i+1}(\ell),s_{i+1}(K)})}{\mbox{KL}(\mathbf{w}^{i+1}(s_{i+1}(\ell)),\mathbf{w}^{i+1}(s_{i+1}(K)))}(\log{T}+3\log{\log{T}})+\frac{C_2(\epsilon)}{d_i^{\beta(\epsilon)}},
\end{align*}
where $C_2(\epsilon)$ and $\beta(\epsilon)$ follow the same definition in~\citet{kveton2015cascading}. Denote $\widetilde{D}_i$ as
\begin{align}
\label{eq:D_i}
  \widetilde{D}_i=  \sum_{\ell=K+1}^L\frac{(1+\epsilon)\Delta^{i+1}_{s_{i+1}(\ell),s_{i+1}(K)}(1+\log(1/\Delta^{i+1}_{s_{i+1}(\ell),s_{i+1}(K)}))}{\mbox{KL}(\mathbf{w}^{i+1}(s_{i+1}(\ell)),\mathbf{w}^{i+1}(s_{i+1}(K)))}(\log{T}+3\log{\log{T}})+\frac{C_2(\epsilon)}{d_i^{\beta(\epsilon)}}.
\end{align}
Summing up all terms, and we have that
\begin{align*}
    \mR(T)\le T(N-1)(L+1)\delta+Tp+\sum_{i=1}^{N-1}d_i+KN\log{\log{T}}+\sum_{i=0}^{N-1}\widetilde{D}_i.
\end{align*}
\end{proof}

\section{Detailed Proofs of Theorem~\ref{thm:lb}}
\label{sec: sup_pf_lb}
\begin{proof}[Proof of Theorem \ref{thm:lb}]
The first step in deriving the minimax lower bound is to construct a randomized `hard instance' as follows. Partition the time horizon $T$ into $N$ blocks and name them $B_1, \ldots, B_N$, where the lengths of first $N-1$ blocks are $\lceil T/N\rceil$ and the length of the last block is $T-(N-1)\lceil T/N\rceil$. In each segment, $L-1$ items follow Bernoulli distribution with probability $1/2$ and only one item follows Bernoulli distribution with probability $1/2+\epsilon$, where $\epsilon$ is a small positive number. Let $\ell_i^*=\arg\max_{\ell\in\mathcal{L}}\mathbf{w}^{i}(\ell)$, i.e, the item with largest click probability during $B_i$. The distributions of the $\ell_i^*$'s are defined as follows:
\begin{itemize}
    \item $\ell_1^*\sim\text{Uniform}(\{1,\ldots, L\})$.
    \item for $i\geq 2$, $\ell_i^*\sim\text{Uniform}(L\setminus \ell_{i-1}^*)$.
\end{itemize}
Note that for this randomized instance, the regret for any policy $\pi$ is
\begin{align*}
    \mathcal{R}^{\pi}(T)=\epsilon(1/2)^{K-1}\mathbb{E}^{\pi}[\sum_{i=1}^{N}\sum_{t\in B_i}\indfunc\{\ell_i^*\not\in\mathcal{A}_t\}].
\end{align*}
The expectation is taken with respect to the policy $\pi$ and this randomized instance. From the above decomposition, we see that to lower bound the regret for any policy $\pi$, it suffices to upper bound $\mathbb{E}^{\pi}[\sum_{i=1}^{N}\sum_{t\in B_i}\indfunc\{\ell_i^*\in\mathcal{A}_t\}]$, the expectation of total number of recommendations to the item with largest click probability. Before we lower bound this quantity, we need some additional notation. Let $P_i^\ell$ be the joint distribution of $\{\mathcal{A}_t, F_t\}_{t\in B_i}$ given the policy $\pi$ and the $\ell$th item being the item with largest click probability, $P_i^0$ be the joint distribution of $\{\mathcal{A}_t, F_t\}_{t\in B_i}$ given the policy $\pi$ and every item follwing the Bernoulli distribution with probability $1/2$. Furthermore, let $\mathbb{E}_i^\ell[\cdot]$ and  $\mathbb{E}_i^0[\cdot]$ as their respective expectations. Let $N_i^\ell$ be the total numbers of appearances of item $\ell$ in the recommendation list during $B_i$. In order to lower bound the target expectation, we need the following lemma.
\begin{lemma}
\label{lemma:num}
For any segment $B_i$ and any $\ell\in\mathcal{L}$, we have
\begin{align*}
    \mathbb{E}_i^\ell[N_i^\ell]\leq \mathbb{E}_i^0[N_i^\ell]+\frac{|B_i|}{2}\sqrt{\E_i^0[N_i^\ell]\log(\frac{1}{1-4\epsilon^2})}.
\end{align*}
\end{lemma}
\begin{proof}[Proof of Lemma~\ref{lemma:num}]
The proof is similar to Lemma A.1 in~\citet{auer2002nonstochastic}. The key difference is we apply the data processing inequality for KL divergence to upper bound the discrepancy of the partial feedback $F_t$'s under different distributions. 
\begin{align*}
    \mathbb{E}_i^\ell[N_i^l]-\mathbb{E}_i^0[N_i^\ell] &\overset{(a)}{\leq} \frac{|B_i|}{2}\left\lVert P_i^\ell - P_i^0\right\rVert_1\\
    &\overset{(b)}{\leq}\frac{|B_i|}{2}\sqrt{2D_{\textup{KL}}(P_i^0||P_i^\ell)}\\
    &=\frac{|B_i|}{2}\sqrt{2\sum_{t\in B_i}D_{\textup{KL}}(P_i^0(F_t|\mathcal{A}_t)||P_i^\ell(F_t|\mathcal{A}_t))}\\
    &\overset{(c)}{\leq}\frac{|B_i|}{2}\sqrt{2\sum_{t\in B_i}D_{\textup{KL}}(P_i^0(\mathbf{Z}_t|\mathcal{A}_t)||P_i^\ell(\mathbf{Z}_t|\mathcal{A}_t))}\\
    &=\frac{|B_i|}{2}\sqrt{\E_i^0[N_i^\ell]\log(\frac{1}{1-4\epsilon^2})},
\end{align*}
where $D_{\textup{KL}}(\cdot)$ is the KL divergence; $(a)$ is due to the boundedness of $N_i^l$; (b) is due to Pinsker's inequality; (c) is due to data processing inequality for KL divergence. 
\end{proof}
Apply Lemma~\ref{lemma:num} for $B_i$ and sum over all items, to get
\begin{align}
    \sum_{\ell\in\mathcal{L}}\E_i^\ell[N_i^\ell]&\leq \sum_{\ell\in\mathcal{L}}\E_i^0[N_i^\ell] +\sum_{\ell\in\mathcal{L}}\frac{|B_i|}{2}\sqrt{\E_i^0[N_i^\ell]\log(\frac{1}{1-4\epsilon^2})}\notag\\
    &\leq |B_i| + \frac{|B_i|}{2}\sqrt{|B_i|L\log(\frac{1}{1-4\epsilon^2})}\label{num_ub},
\end{align}
where the last inequality is due to $\sum_{\ell\in\mathcal{L}}\E_i^0[N_i^\ell]=|B_i|$ and Jensen's inequality. Then we are able to lower bound the regret for any policy $\pi$.
\begin{align*}
    \mathcal{R}^{\pi}(T)&=\epsilon(1/2)^{K-1}\left(T-\mathbb{E}^{\pi}[\sum_{i=1}^{N}\sum_{t\in B_i}\indfunc\{\ell_i^*\in\mathcal{A}_t\}]\right)\\
    &\overset{(a)}{\geq}(1/2)^{K-1}\epsilon\left(T-\tfrac{1}{L-1}(\sum_{i=1}^{N}|B_i|+\frac{|B_i|}{2}\sqrt{L|B_i|\log{\tfrac{1}{1-4\epsilon^2}}}\right)\\
                        &=(1/2)^{K-1}\left(\epsilon T - \tfrac{\epsilon T}{L-1} -\tfrac{\epsilon T}{2(L-1)}\sqrt{\frac{LT}{N}\log{\tfrac{1}{1-4\epsilon^2}}}\right)\\
                        &\overset{(b)}{\geq}(1/2)^{K-1}\left(\tfrac{\epsilon T}{2}-\tfrac{\epsilon^2 T}{K-1}\sqrt{\tfrac{LT}{N}\log{\tfrac{4}{3}}}\right).
\end{align*}
where $(a)$ is due to inequality~(\ref{num_ub}), and $(b)$ holds by $L\geq 3$, $4\epsilon^2\leq\frac{1}{4}$ and $\log{\frac{1}{1-x}}\leq 4\log(\frac{4}{3})x$ for all $x\in[0, \frac{1}{4}]$. Finally, setting $\epsilon=\frac{L-1}{4\sqrt{TL\log{(\frac{4}{3})}}}$ finishes the proof. 
\end{proof}